%% file: main.tex
\documentclass[sigconf]{acmart}
\usepackage{hyperref}
\usepackage{url}
\usepackage{multirow}
\usepackage{graphicx}
\usepackage{subcaption}
\usepackage{caption}
\usepackage[normalem]{ulem}
\usepackage{amsmath}
\usepackage{amsthm}
\usepackage[ruled,vlined]{algorithm2e}
\usepackage{algorithmic}
\usepackage{bm}
\usepackage{float}

\usepackage{tikz}
\usetikzlibrary{shapes, arrows.meta, positioning, calc, shadows.blur, patterns}
\usepackage{pgfplots}
\pgfplotsset{compat=1.18}

\AtBeginDocument{%
  }

\copyrightyear{2026}
\acmYear{2026}
\setcopyright{cc}
\setcctype{by}
\acmConference[KDD '26]{Proceedings of the 32nd ACM SIGKDD Conference on Knowledge Discovery and Data Mining V.2}{August 09--13, 2026}{Jeju Island, Republic of Korea}
\acmBooktitle{Proceedings of the 32nd ACM SIGKDD Conference on Knowledge Discovery and Data Mining V.2 (KDD '26), August 09--13, 2026, Jeju Island, Republic of Korea}
\acmDOI{10.1145/3770855.3817983}
\acmISBN{979-8-4007-2259-2/2026/08}

\begin{document}

\title{Defending against Model Extraction for GNNs with Model Reprogramming}

\author{Yan Wen}
\orcid{0009-0002-6425-5056}
\affiliation{%
  \institution{University of Maryland, College Park}
  \department{Department of Computer Science}
  \city{College Park}
  \state{MD}
  \country{USA}
}
\email{ywen1@umd.edu}

\author{Zhenyi Wang}
\orcid{0000-0002-2780-9446}
\affiliation{%
  \institution{University of Central Florida}
  \department{Department of Computer Science and Institute of Artificial Intelligence}
  \city{Orlando}
  \state{FL}
  \country{USA}
}
\email{zhenyi.wang@ucf.edu}

\author{Heng Huang}
\orcid{0000-0002-3483-8333}
\affiliation{%
  \institution{University of Maryland, College Park}
  \department{Department of Computer Science}
  \city{College Park}
  \state{MD}
  \country{USA}
}
\email{heng@umd.edu}

\renewcommand{\shortauthors}{Yan Wen, Zhenyi Wang, and Heng Huang}



\begin{abstract}
Graph Neural Networks (GNNs) serve as the backbone for high-stakes applications in Machine-Learning-as-a-Service (MLaaS). Still, their black-box deployment exposes them to Model Extraction (ME) attacks, in which adversaries steal intellectual property by querying APIs. Existing defenses suffer from a critical ``Euclidean bias'': they transfer image-based strategies (e.g., random noise) to graphs, ignoring the complex topological dependencies between nodes, which often results in severe utility degradation. Passive methods like watermarking also fail to prevent theft in real time. To bridge this gap, we propose \textbf{GraphRP} (\textbf{Graph} \textbf{R}eprogramming \textbf{P}rotection), a proactive defense framework that repurposes \textit{Model Reprogramming} for security. Unlike static perturbations, GraphRP introduces a Structure-Aware Gating Mechanism driven by learnable topological prototypes. This creates a dynamic ``structural firewall'' that selectively modulates the model's decision boundary: it preserves fidelity for benign queries residing on the training manifold, while maximizing the Fisher Information along the perturbation direction for adversarial queries. Under standard assumptions (bounded loss, optimal attacker, and local second-order approximation), we prove a lower bound on the attacker's estimation error that increases with the structural sensitivity of the reprogramming noise. Extensive experiments on both hard-label and soft-label ME attacks demonstrate that GraphRP significantly degrades attack effectiveness while preserving benign utility.
\end{abstract}

\maketitle
\input{1-introduction}
\input{2-relatedworks}
\input{3-preliminary}
\input{4-method}
\input{5-analysis}
\input{6-experiments}
\input{7-conclusions}


\begin{acks}
This work was partially supported by NSF IIS-2347592, IIS-2348169, DBI-2405416, CCF-2348306, CNS-2347617, and RISE-2536663.
\end{acks}

\clearpage
\bibliographystyle{ACM-Reference-Format}
\balance
\bibliography{sample-base}
\clearpage

\appendix
\input{appendix}
\clearpage

\end{document}

%% file: 1-introduction.tex
\section{Introduction}
\label{sec:intro}

Graph Neural Networks (GNNs) \cite{kipf2016semi, hamilton2017inductive, velivckovic2017graph} have emerged as the de facto standard for modeling non-Euclidean data, serving as the backbone for high-stakes applications in recommendation systems \cite{wu2022graph}, financial fraud detection \cite{liu2021pick}, and social network analysis \cite{li2023survey}. 
Given the high costs associated with data collection and model training, these proprietary models are increasingly deployed as cloud-based API services, a paradigm known as \textit{Machine-Learning-as-a-Service} (MLaaS) \cite{liu2023unlearnable, long2022pre}. 
However, this black-box access exposes them to \textbf{Model Extraction (ME) attacks}, where an adversary systematically queries the API to train a surrogate model that replicates the victim's functionality \cite{shen2022model, wu2022model, zhuang2024unveiling}. 
Unlike attacks on independent and identically distributed (i.i.d.) image data, ME attacks on GNNs exploit the complex topological dependencies between nodes, allowing attackers to infer global graph properties from limited queries. This renders GNNs uniquely vulnerable to intellectual property (IP) theft and subsequent adversarial exploitation \cite{defazio2019adversarial}. 

While prior work has made progress in defending against model extraction, existing solutions have critical limitations when applied to GNNs.
First, \textbf{passive defenses}, such as watermarking \cite{downer2025watermarking, jia2021entangled}, focus on ownership verification rather than prevention. These methods can only identify a stolen model \textit{after} the theft has occurred, by which time the intellectual property is already compromised.
Second, \textbf{active defenses} transferred from the image domain fail to account for the specific properties of graph data. Most existing active methods attempt to confuse attackers by adding random noise or performing complex calculations during inference \cite{orekondy2019prediction, kariyappa2020defending, mazeika2022steer, wang2023defending}. 
However, these approaches face two major problems with graphs: 
(1) \textit{Utility Degradation}: Unlike pixels in an image, nodes in a graph are connected. Adding indiscriminate noise to a single node propagates errors to its neighbors during message passing, severely harming the model's accuracy for benign users. 
(2) \textit{High Latency}: Many defense methods require expensive optimization steps for every query. For real-world web applications like recommendation systems, this extra delay is unacceptable.

These challenges lead us to the central research question of this work:
\textit{\textbf{RQ:} How can we design an active defense for GNNs that effectively prevents model extraction without hurting the accuracy for benign users, while keeping inference fast?}

To address this, we propose \textbf{GraphRP} (\textbf{Graph} \textbf{R}eprogramming \textbf{P}rotection), a novel active defense mechanism based on \textbf{Model Reprogramming} \cite{chen2024model, jing2023deep}. 
While reprogramming is traditionally used to adapt pre-trained models to new tasks, we repurpose it for security. 
Our key insight is to treat defense as a "conditional task": the model should behave normally for benign queries but essentially "malfunction" for adversarial ones.
Unlike random noise injection, GraphRP optimizes a set of \textbf{learnable, layer-wise perturbations} that act as a "structural firewall."
Crucially, we introduce a \textbf{Structure-Aware Gating Mechanism} that dynamically modulates these perturbations based on the topological signature of the input graph. 
By learning "benign structural prototypes," our method ensures that defensive noise is only activated when the query distribution deviates from the benign distribution (OOD), thereby preserving utility for benign users.

In summary, our main contributions are fourfold:
\begin{itemize}
    \item \textbf{Framework:} We propose GraphRP, the first defense framework that leverages model reprogramming to protect GNNs against model extraction, avoiding the need for full model retraining.
    \item \textbf{Structure-Awareness:} We introduce a prototype-based gating mechanism that exploits graph topological invariants (e.g., spectral features) to distinguish benign queries from extraction attacks, solving the utility-defense trade-off.
    \item \textbf{Theoretical Guarantee:} Under standard assumptions (bounded loss, an optimal attacker, and a local second-order approximation), we show that our method maximizes a lower bound on the attacker's estimation error by exploiting the Fisher Information of the target GNN.
    \item \textbf{Empirical Effectiveness:} Extensive experiments on diverse benchmarks demonstrate that GraphRP reduces clone model accuracy by up to $17\%$ compared to state-of-the-art baselines while maintaining high utility and low inference latency.
\end{itemize}

%% file: 2-relatedworks.tex
\section{Related Work}

\subsection{Model Extraction}
\textbf{Model Extraction (ME)}, also referred to as \textit{model stealing}, denotes the adversarial process of replicating the functionality of a target model by extracting its parameters or approximating its decision boundary \cite{orekondy2019knockoff, papernot2017practical, truong2021data, oliynyk2023know}.
Attackers may target exact parameters \cite{reith2019efficiently}, hyperparameters \cite{wang2018stealing}, or neural architectures \cite{oh2019towards}; this work, however, focuses on \textit{behavioral cloning}—training a surrogate that mimics the target's outputs without accessing its internals \cite{defazio2019adversarial}.

ME attacks are conventionally categorized based on the data available to the adversary. \textbf{Data-Based Model Extraction (DBME)} \cite{kariyappa2021maze, correia2018copycat, papernot2017practical} assumes the attacker possesses a dataset distributionally similar to the victim's training data. By querying the target model with these inputs, the attacker collects labeled pairs to train the surrogate in a supervised manner. Conversely, \textbf{Data-Free Model Extraction (DFME)} \cite{kariyappa2021maze, truong2021data} represents a more sophisticated threat where no real data is available. In this scenario, the attacker must generate synthetic queries to simultaneously explore the high-dimensional input space and learn the model's behavior, often utilizing generative models to maximize query efficiency.

The difficulty of extraction is further dictated by the granularity of the API's output. In \textbf{soft-label attacks}, the victim returns a full probability distribution (logits), providing rich information about the decision boundary \cite{hinton2015distilling}. In contrast, \textbf{hard-label attacks} return only the top-1 class prediction. This setting significantly restricts the information gain per query, forcing attackers to rely on label-only learning techniques that are typically less sample-efficient and harder to optimize.

Recent research has demonstrated that GNNs are uniquely vulnerable to these attacks due to the leakage of structural information \cite{defazio2019adversarial, wu2022model, shen2022model, zhuang2024unveiling}. Unlike image models, GNN extraction reconstructs not only node labels but also the underlying graph topology. \citet{wu2022model} proposed using discrete graph structure learning to infer a substitute graph from node attributes, while \citet{shen2022model} utilized k-Nearest Neighbors (kNN) to initialize surrogate structures. More recently, \citet{zhuang2024unveiling} proposed StealGNN, a data-free extraction attack for GNNs that leverages a generative model to synthesize queries without access to any real graph data, demonstrating that structural information leaks even under strict black-box constraints. Furthermore, as the field moves toward large-scale pre-training, \citet{xu2025systematic} have shown that even Graph Foundation Models are susceptible to extraction, emphasizing the urgent need for robust defense mechanisms in the graph domain.

\subsection{Model Extraction Defense}

Existing defenses against model extraction can be broadly categorized into passive and active strategies, depending on whether they attempt to detect the attack or prevent it entirely. 

\textbf{Passive Defenses} focus on monitoring and detecting extraction attempts without modifying the model's inference behavior \cite{jia2021entangled, szyller2021dawn, maini2021dataset}. These methods typically analyze query patterns, distribution shifts, and timing to identify suspicious activities distinct from normal user behavior. Standard techniques involve sophisticated logging and anomaly detection pipelines. In the specific context of GNNs, recent approaches, such as those by \citet{downer2025watermarking}, have proposed embedding watermarks directly into GNN explanations to verify ownership. While they provide evidence for legal recourse after the fact, they fail to prevent intellectual property theft in the first place.

\textbf{Active Defenses}, in contrast, aim to prevent model extraction (ME) attacks by proactively modifying the model's responses or access mechanisms \cite{orekondy2019prediction, kariyappa2020defending, kariyappa2021protecting, mazeika2022steer, wang2023defending}. Primary strategies include \textit{Prediction Obfuscation} \cite{orekondy2019prediction}, which alters output logits to reduce the information gain per query, and \textit{Perturbation Techniques} \cite{kariyappa2020defending, kariyappa2021protecting, wang2023defending}, which introduce stochastic noise to degrade the fidelity of the extracted surrogate. Other approaches, such as \textit{Query Limitation} \cite{mazeika2022steer}, restrict the rate of queries or dynamically adjust responses based on behavior analysis \cite{wang2023defending, zhuang2024unveiling}.

Despite these advancements, designing defenses for GNNs presents unique challenges that existing methods fail to address. We term this failure mode the \textit{Euclidean bias}. Graph data is non-Euclidean, containing complex dependencies between node features and topological structures. Most existing active defenses are designed for independent and identically distributed (i.i.d.) image data; when applied to graphs, they often disrupt the message-passing mechanism, leading to severe utility degradation for benign users. A recent work, ADAGE \cite{xu2025adage}, takes a step toward GNN-specific active defense by monitoring the diversity of incoming queries across sessions and reactively blocking requests that exhibit extraction-like patterns. However, ADAGE operates in a \textit{stateful} manner, requiring persistent query-log monitoring across multiple API calls, which introduces significant infrastructure overhead and makes it inapplicable in truly stateless black-box deployments. Furthermore, its reactive blocking strategy depends on accumulating sufficient query history before a defense is triggered, leaving an exploitable window for early-stage attacks. In contrast, GraphRP is entirely \textit{stateless}: it requires no query logging or cross-session state, and proactively poisons every OOD query at inference time based solely on the structural signature of the current input. These fundamental architectural differences make direct experimental comparison between ADAGE and GraphRP infeasible, as they operate under incompatible deployment assumptions.

Our work bridges this gap by proposing a novel defense pipeline that emphasizes graph-structure sensitivity. We introduce adaptive, layer-wise perturbations that exploit the unique construction of graph-based models. Unlike prior works, we leverage \textbf{Model Reprogramming} as an active defense mechanism. To the best of our knowledge, this is the first study to explore Model Reprogramming specifically to defend against GNN model extraction, offering a proactive solution that reduces surrogate fidelity while preserving benign utility.

\subsection{Model Reprogramming}

\textbf{Model Reprogramming (MR)} is a parameter-efficient learning paradigm that repurposes pre-trained machine learning models for target tasks in different domains without modifying the original model weights \cite{chen2024model}. Initially introduced in the context of adversarial machine learning \cite{elsayed2018adversarial}, MR demonstrates that a fixed model can be "steered" to perform a new function simply by learning a transformation function (or perturbation mask) applied to the input data. This approach has since evolved into a powerful tool for cross-domain adaptation, offering a resource-efficient alternative to traditional transfer learning, where full model fine-tuning is computationally prohibitive.

In the graph domain, MR has gained significant traction under the umbrella of "Graph Prompting" and "Deep Graph Reprogramming." \citet{jing2023deep} proposed \textit{Deep Graph Reprogramming}, which aligns pre-trained GNNs with novel downstream tasks by optimizing a set of learnable structural perturbations. Similarly, \citet{sun2023all} demonstrated that reprogramming-based prompting can unify multi-task learning on graphs by modifying the input graph topology and features rather than the model architecture. 

Our work introduces a paradigm shift in the application of MR. While prior works \cite{jing2023deep, sun2023all} introduce reprogramming to enhance model utility for new tasks, we leverage it as a defensive mechanism. We reprogram the victim GNN to maintain utility for the original task (benign users) while simultaneously degrading utility for the specific task of "model extraction" (adversaries), effectively turning the model's plasticity into a security feature.

\textbf{Distinction from Watermarking.} 
It is crucial to distinguish GraphRP from watermarking techniques \cite{jia2021entangled}. 
Watermarking is a \textit{passive, post-hoc} mechanism that embeds a signature into the model to verify ownership \textit{after} theft has occurred. It does not prevent the attacker from using the stolen model. 
In contrast, GraphRP is an \textit{active, real-time} defense. It fundamentally degrades the quality of the extracted model, rendering the stolen copy useless. 
While watermarking seeks legal recourse, GraphRP seeks to render the surrogate functionally obsolete.

%% file: 3-preliminary.tex

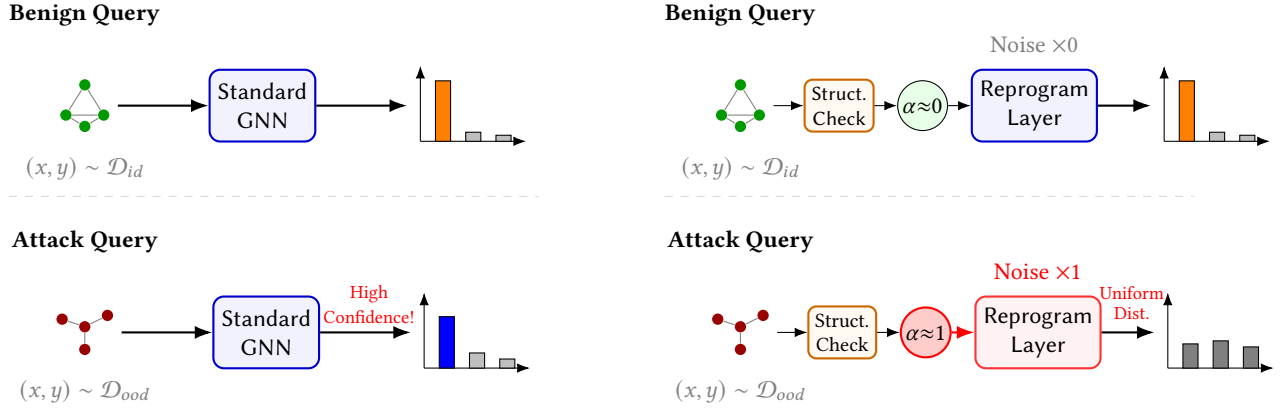
\begin{figure*}[t]
    \centering
    \tikzset{
        font=\sffamily\normalsize,
        >=LaTeX,
        packet/.style={draw=black!50, thick, fill=white, rounded corners=2pt, minimum size=1.5em, blur shadow},
        model/.style={draw=blue!80!black, thick, fill=blue!5, rounded corners=4pt, minimum height=3em, minimum width=4em, align=center},
        noise/.style={draw=red!80!black, thick, fill=red!10, dashed, rounded corners=3pt, align=center},
        module/.style={draw=orange!80!black, thick, fill=orange!5, rounded corners=3pt, align=center},
        node_benign/.style={circle, fill=green!60!black, inner sep=1.5pt},
        node_attack/.style={circle, fill=red!60!black, inner sep=1.5pt},
        edge_style/.style={thin, gray},
        bar_base/.style={ybar, bar width=4pt, draw=none, align=center},
    }

    \begin{subfigure}[b]{0.48\textwidth}
        \centering
        \begin{tikzpicture}[node distance=0.8cm and 0.5cm]
            
            \node[font=\bfseries] (label_benign) at (-2, 1.2) {Benign Query};
            
            \node (g_benign) at (-2, 0) {
                \tikz[scale=0.50, baseline=-0.5ex]{
                    \node[node_benign] (a) at (0,0.5) {}; \node[node_benign] (b) at (-0.5,-0.3) {};
                    \node[node_benign] (c) at (0.5,-0.3) {}; \node[node_benign] (d) at (0, -0.6) {};
                    \draw[edge_style] (a)--(b)--(c)--(a) (b)--(d)--(c);
                }
            };
            \node[below=0.1cm of g_benign, font=\normalsize, color=gray] {$(x,y) \sim \mathcal{D}_{id}$};

            \node[model, right=1.2cm of g_benign] (gnn_benign) {Standard\\GNN};
            
            \node[right=1.2cm of gnn_benign] (out_benign) {
                \tikz[scale=0.40, baseline=0.5ex]{
                    \draw[->] (0,0) -- (3.5,0); \draw[->] (0,0) -- (0,2.5);
                    \draw[fill=orange] (0.5,0) rectangle (1.0, 2.0); 
                    \draw[fill=gray!50] (1.5,0) rectangle (2.0, 0.3);
                    \draw[fill=gray!50] (2.5,0) rectangle (3.0, 0.2);
                }
            };
            
            \draw[->, thick] (g_benign) -- (gnn_benign);
            \draw[->, thick] (gnn_benign) -- (out_benign);

            \draw[dashed, gray!30] (-3, -1.2) -- (4, -1.2);

            \node[font=\bfseries] (label_attack) at (-2, -1.8) {Attack Query};

            \node (g_attack) at (-2, -3.0) {
                \tikz[scale=0.50, baseline=-0.5ex]{
                    \node[node_attack] (a) at (0,0) {}; \node[node_attack] (b) at (0.6,0.3) {};
                    \node[node_attack] (c) at (-0.6,0.2) {}; \node[node_attack] (d) at (0,-0.6) {};
                    \draw[edge_style] (a)--(b) (a)--(c) (a)--(d); 
                }
            };
            \node[below=0.1cm of g_attack, font=\normalsize, color=gray] {$(x,y) \sim \mathcal{D}_{ood}$};

            \node[model, right=1.2cm of g_attack] (gnn_attack) {Standard\\GNN};

            \node[right=1.2cm of gnn_attack] (out_attack) {
                \tikz[scale=0.40, baseline=0.5ex]{
                    \draw[->] (0,0) -- (3.5,0); \draw[->] (0,0) -- (0,2.5);
                    \draw[fill=blue] (0.5,0) rectangle (1.0, 1.7); 
                    \draw[fill=gray!50] (1.5,0) rectangle (2.0, 0.5);
                    \draw[fill=gray!50] (2.5,0) rectangle (3.0, 0.3);
                }
            };

            \draw[->, thick] (g_attack) -- (gnn_attack);
            \draw[->, thick] (gnn_attack) --  node[above, font=\footnotesize, red, xshift=0.05cm] {\shortstack{High\\Confidence!}}   (out_attack);

        \end{tikzpicture}
    \end{subfigure}
    \hfill
    \begin{subfigure}[b]{0.48\textwidth}
        \centering
        \begin{tikzpicture}[node distance=0.8cm and 0.5cm]

            \node[font=\bfseries] (label_benign_b) at (-2, 1.2) {Benign Query};

            \node (g_benign_b) at (-2, 0) {
                \tikz[scale=0.50, baseline=-0.5ex]{
                    \node[node_benign] (a) at (0,0.5) {}; \node[node_benign] (b) at (-0.5,-0.3) {};
                    \node[node_benign] (c) at (0.5,-0.3) {}; \node[node_benign] (d) at (0, -0.6) {};
                    \draw[edge_style] (a)--(b)--(c)--(a) (b)--(d)--(c);
                }
            };
            
            \node[module, right=0.4cm of g_benign_b, scale=0.85] (struct_benign) {Struct.\\Check};
            \node[below=0.1cm of g_benign_b, font=\normalsize, color=gray] {$(x,y) \sim \mathcal{D}_{id}$};
            
            \node[draw, circle, inner sep=1pt, right=0.3cm of struct_benign, fill=green!10] (gate_benign) {\normalsize $\alpha{\approx}0$};
            
            \node[model, right=0.3cm of gate_benign, align=center] (gnn_benign_b) {Reprogram\\Layer};

            \node[right=0.7cm of gnn_benign_b] (out_benign_b) {
                \tikz[scale=0.40, baseline=0.5ex]{
                    \draw[->] (0,0) -- (3.5,0); \draw[->] (0,0) -- (0,2.5);
                    \draw[fill=orange] (0.5,0) rectangle (1.0, 2.0); 
                    \draw[fill=gray!50] (1.5,0) rectangle (2.0, 0.3);
                    \draw[fill=gray!50] (2.5,0) rectangle (3.0, 0.2);
                }
            };

            \draw[->] (g_benign_b) -- (struct_benign);
            \draw[->] (struct_benign) -- (gate_benign);
            \draw[->] (gate_benign) -- (gnn_benign_b);
            \draw[->, thick] (gnn_benign_b) -- (out_benign_b);
            
            \node[above=0.1cm of gnn_benign_b, font=\normalsize, color=gray] {Noise $\times 0$};

            \draw[dashed, gray!30] (-3, -1.2) -- (4, -1.2);

            \node[font=\bfseries] (label_attack_b) at (-2, -1.8) {Attack Query};

            \node (g_attack_b) at (-2, -3.0) {
                \tikz[scale=0.50, baseline=-0.5ex]{
                    \node[node_attack] (a) at (0,0) {}; \node[node_attack] (b) at (0.6,0.3) {};
                    \node[node_attack] (c) at (-0.6,0.2) {}; \node[node_attack] (d) at (0,-0.6) {};
                    \draw[edge_style] (a)--(b) (a)--(c) (a)--(d);
                }
            };
            \node[below=0.1cm of g_attack_b, font=\normalsize, color=gray] {$(x,y) \sim \mathcal{D}_{ood}$};

            \node[module, right=0.4cm of g_attack_b, scale=0.85, font=\normalsize] (struct_attack) {Struct.\\Check};

            \node[draw, circle, thick, inner sep=1pt, right=0.3cm of struct_attack, fill=red!20, draw=red] (gate_attack) {\normalsize $\alpha{\approx}1$};

            \node[model, right=0.3cm of gate_attack, fill=red!5, draw=red!80] (gnn_attack_b) {Reprogram\\Layer};

            \node[right=0.7cm of gnn_attack_b] (out_attack_b) {
                \tikz[scale=0.40, baseline=0.5ex]{
                    \draw[->] (0,0) -- (3.5,0); \draw[->] (0,0) -- (0,2.5);
                    \draw[fill=gray] (0.5,0) rectangle (1.0, 0.8); 
                    \draw[fill=gray] (1.5,0) rectangle (2.0, 0.9);
                    \draw[fill=gray] (2.5,0) rectangle (3.0, 0.7);
                }
            };

            \draw[->] (g_attack_b) -- (struct_attack);
            \draw[->] (struct_attack) -- (gate_attack);
            \draw[->, thick, red] (gate_attack) -- (gnn_attack_b);
            \draw[->, thick] (gnn_attack_b) -- node[above, font=\footnotesize, red, xshift=0.05cm, yshift=0.05cm] {\shortstack{Uniform \\Dist.}} (out_attack_b);

            \node[above=0.1cm of gnn_attack_b, font=\normalsize, color=red] {Noise $\times 1$};

        \end{tikzpicture}
    \end{subfigure}
    
    \caption{\textbf{Overview of the Proposed Defense Framework.} Comparison between a standard GNN pipeline and our approach. While the standard GNN pipeline (left) leaks high-confidence information for both benign and attack queries, GraphRP (right) dynamically modulates the output based on structural legitimacy. The top row shows utility preservation for benign users (orange bars), while the bottom row demonstrates defense activation for OOD attackers (grey/flat bars).}
    \label{fig:pipeline}
\end{figure*}

\section{Preliminaries}
\subsection{Notations and Problem Setup}

\noindent\textbf{Graph Notation.} Let $G = (V, E)$ denote a graph, where $V$ is the set of $n$ nodes (vertices) and $E$ is the set of edges. A graph sample can be represented as a tuple $x = (A, X)$, where $A \in \{0, 1\}^{n \times n}$ is the adjacency matrix representing the graph structure, and $X \in \mathbb{R}^{n \times d}$ is the node feature matrix, with $d$ denoting the feature dimension.

\noindent\textbf{Data Distributions.} We consider a graph classification task where data is sampled from a ground-truth distribution $\mathcal{D}$. A dataset is denoted as $D = \{(x_i, y_i)\}_{i=1}^N$, where $y_i \in Y$ is the label. 
Crucially for Model Extraction (ME) defense, we distinguish between two data regimes:
\begin{itemize}
    \item \textbf{In-Distribution (ID):} The private distribution $\mathcal{D}_{id}$ on which the victim model is trained and expected to perform well. The defender possesses a private dataset $D_{train} \sim \mathcal{D}_{id}$.
    \item \textbf{Out-of-Distribution (OOD):} The public or synthetic distribution $\mathcal{D}_{ood}$ used by the attacker to query the model. This includes surrogate datasets (in Data-Based ME) or synthetic graphs generated by a generator (in Data-Free ME).
\end{itemize}
In our setting, we assume the query data $x_q$ comes from $\mathcal{D}_{ood}$, such that $\mathcal{D}_{id} \cap \mathcal{D}_{ood} = \emptyset$ in terms of exact samples, though they may share feature semantics.

\subsection{Model Framework}

\noindent\textbf{Victim Model (Target).} The defender deploys a pre-trained Graph Neural Network, denoted as $T: \mathcal{G} \to \mathcal{Y}$, parameterized by $\theta_T$. This model maps an input graph $x \in \mathcal{G}$ to a probability distribution over classes $Y$. We assume $T$ is well-trained on the private distribution $\mathcal{D}_{id}$ and serves as an oracle $\mathcal{O}$ accessible via an API.

\noindent\textbf{Surrogate Model (Clone).} The attacker aims to train a surrogate model $C: \mathcal{G} \to \mathcal{Y}$, parameterized by $\theta_C$. The architecture of $C$ is generally unknown to the defender and may differ significantly from $T$. The attacker's objective is to estimate $\theta_C$ such that $C(x) \approx T(x)$ for all $x$ in the domain of interest with similar functionality.

\noindent\textbf{Query Generator.} In the Data-Free Model Extraction (DFME) setting, the attacker utilizes a generator $f_{gen}: \mathcal{Z} \to \mathcal{G}_{ood}$ to synthesize query graphs. Here, $\mathcal{Z}$ represents a latent space (e.g., Gaussian noise), and $f_{gen}$ constructs node features and adjacency matrices that are structurally valid but distributionally distinct from the private data $\mathcal{D}_{id}$.

\subsection{Attacker's Goal and Knowledge}

\noindent\textbf{Attacker's Knowledge.} We assume a black-box setting where the attacker has no access to the victim model's parameters $\theta_T$, gradients, or the private training data $\mathcal{D}_{id}$. The attacker can only query the API with inputs $x_q \sim \mathcal{D}_{ood}$ and observe the output $y_q$, as illustrated in Figure~\ref{fig:pipeline}. 
In the \textit{soft-label} setting, $y_q$ is a full probability vector; in the \textit{hard-label} setting, $y_q$ is the top-1 predicted class.

\noindent\textbf{Optimization Objective.} The attacker aims to train a surrogate model $C(\cdot; \theta_C)$ that mimics the behavior of the target model $T$. Since the attacker lacks access to $\mathcal{D}_{id}$, they minimize the divergence between the two models over the available query distribution $\mathcal{D}_{ood}$. Formally, the attacker minimizes the Kullback-Leibler (KL) divergence:
\begin{align}
    \label{eqn:attacker_obj}
    \min_{\theta_C} \ \mathbb{E}_{x \sim \mathcal{D}_{ood}} \left[ D_{\text{KL}} \left( T(x; \theta_T) \parallel C(x; \theta_C) \right) \right],
\end{align}
where $T(x; \theta_T)$ represents the soft labels returned by the victim (or approximated from hard labels). By minimizing this loss, the surrogate learns to replicate the decision boundary of the victim on the query set.

\noindent\textbf{Ultimate Goal.} While the optimization occurs on $\mathcal{D}_{ood}$, the attacker's ultimate goal is to obtain a model that generalizes well to the private task. That is, the surrogate $C$ should achieve high accuracy on the private test set $D_{id}^{test}$.

\subsection{Defender's Goal and Knowledge}

\noindent\textbf{Defender's Knowledge.} The defender has full white-box access to the target model $T$ and the private training set $D_{id}$. Crucially, we assume the defender knows that attack queries are likely to be out-of-distribution (OOD) relative to the private data \cite{wang2021zero, wang2023defending}. This assumption allows the defender to leverage an auxiliary OOD dataset $\mathcal{D}_{ood}^{def}$ (different from the attacker's queries) to calibrate the defense.

\noindent\textbf{Defense Objectives.} The defender has three competing objectives:
\begin{enumerate}
    \item \textbf{Utility Preservation:} The target model must maintain high accuracy for benign users (ID queries), therefore, 
    \begin{align}
        \mathbb{E}_{(x,y) \sim \mathcal{D}_{id}} \left[ \mathbb{I}( \arg\max T(x) = y ) \right] \text{ is maximized.}
    \end{align}
    \item \textbf{Extraction Prevention:} The target model should provide uninformative outputs for attack queries (OOD) to prevent the surrogate from learning the true decision boundary.
    \item \textbf{Efficiency:} The defense mechanism must be computationally lightweight, avoiding expensive test-time optimization that would increase latency for real-time services.
\end{enumerate}



%% file: 4-method.tex
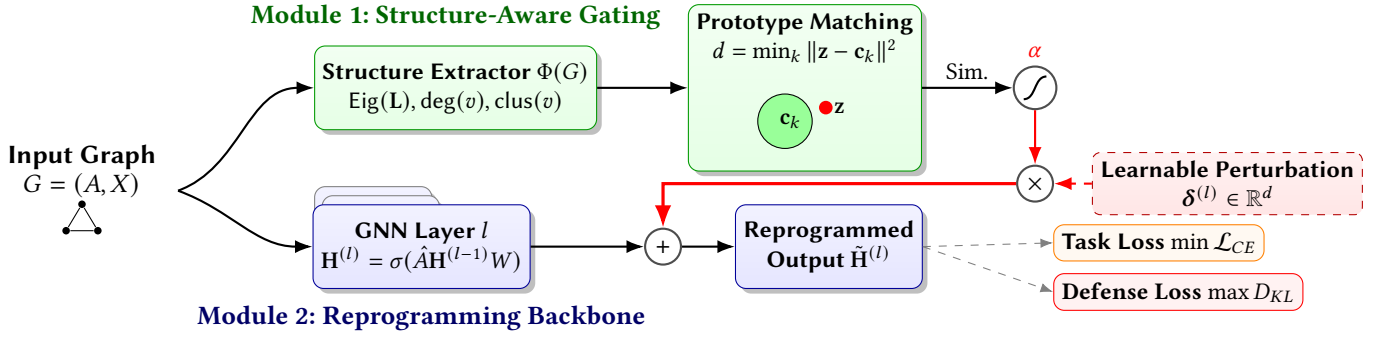
\begin{figure*}[t]
    \centering
    \begin{tikzpicture}[
        font=\sffamily\normalsize,
        >=LaTeX,
        node distance=1.2cm and 1.0cm,
        box_module/.style={
            draw=blue!60!black, 
            top color=blue!5, 
            bottom color=blue!10, 
            rounded corners=4pt, 
            minimum height=3.5em, 
            minimum width=4.9em, 
            align=center,
            drop shadow
        },
        box_struct/.style={
            draw=green!60!black, 
            top color=green!5, 
            bottom color=green!10, 
            rounded corners=4pt, 
            minimum height=3.5em, 
            align=center,
            drop shadow
        },
        box_param/.style={
            draw=red!70!black, 
            top color=red!5, 
            bottom color=red!10, 
            dashed, 
            rounded corners=3pt, 
            align=center,
            minimum height=2.5em
        },
        op/.style={
            circle, 
            draw=black!70, 
            fill=white, 
            thick, 
            inner sep=1pt, 
            minimum size=1.5em
        },
        tensor/.style={
            draw=black!40, 
            fill=gray!5, 
            shape=rectangle, 
            minimum width=0.5cm, 
            minimum height=1.5cm, 
            outer sep=3pt
        }
    ]

    \node (input_g) at (0,0) {
        \begin{tabular}{c}
             \large \textbf{Input Graph} \\
             \large $G=(A, X)$ \\
             \tikz[scale=0.35, baseline=-0.5ex]{
                \node[circle,fill=black,inner sep=1pt] (a) at (0,0.8) {}; 
                \node[circle,fill=black,inner sep=1pt] (b) at (-0.6,-0.2) {};
                \node[circle,fill=black,inner sep=1pt] (c) at (0.6,-0.2) {};
                \draw (a)--(b)--(c)--(a);
             }
        \end{tabular}
    };

    
    \node[box_struct, right=1.8cm of input_g, anchor=south west, yshift=0.8cm] (extractor) {
        \textbf{Structure Extractor} $\Phi(G)$ \\
        \normalsize $\text{Eig}(\mathbf{L}), \text{deg}(v), \text{clus}(v)$
    };
    
    \node[box_struct, right=1.2cm of extractor] (prototypes) {
        \textbf{Prototype Matching} \\
        \normalsize $d = \min_k \| \mathbf{z} - \mathbf{c}_k \|^2$ \\
        \tikz[scale=0.18, baseline=0ex]{
            \draw[fill=green!40] (0,0) circle (2); \node at (-1,0) {$\mathbf{c}_k$};
            \fill[red] (3,1) circle (0.5); \node at (3,1) {$\mathbf{z}$};
        }
    };
    
    \node[op, right=1.2cm of prototypes, label=above:{\color{red}\bfseries $\alpha$}] (gate) {
        \tikz[scale=0.15, baseline=-0.5ex]{\draw[thick] (-1,-1) to[out=0,in=180] (1,1);}
    };

    \draw[->, thick] (input_g.east) to[out=20, in=180] (extractor.west);
    \draw[->, thick] (extractor.east) -- (prototypes.west);
    \draw[->, thick] (prototypes.east) -- node[above, font=\normalsize] {Sim.} (gate.west);

    
    \node[box_module, right=1.8cm of input_g, yshift=-0.5cm, fill=white, draw=gray] (layer1) {};
    \node[box_module, at=(layer1), shift={(0.10,-0.12)}, fill=white, draw=gray] (layer2) {};
    \node[box_module, at=(layer1), shift={(0.64,-0.24)}] (gnn_layer) {
        \textbf{GNN Layer} $l$ \\
         $\mathbf{H}^{(l)} = \sigma(\hat{A} \mathbf{H}^{(l-1)} W)$
    };
        
    \node[op, right=1.5cm of gnn_layer] (add_node) {$+$};
    
    \node[box_module, right=0.7cm of add_node] (output) {
        \textbf{Reprogrammed} \\
        \textbf{Output} $\tilde{\mathbf{H}}^{(l)}$
    };

    \draw[->, thick] (input_g.east) to[out=-20, in=180] (gnn_layer.west);
    \draw[->, thick] (gnn_layer.east) -- (add_node.west);
    \draw[->, thick] (add_node.east) -- (output.west);


    \node[op, below=0.72cm of gate] (mult_node) {$\times$};
    \node[box_param, right=0.5cm of mult_node] (noise_param) {
        \textbf{Learnable Perturbation} \\
        $\boldsymbol{\delta}^{(l)} \in \mathbb{R}^d$
    };
    
    \draw[->, thick, red] (gate.south) -- (mult_node.north); 
    \draw[->, thick, red, dashed] (noise_param.west) -- (mult_node.east); 
    
    \draw[->, very thick, red] (mult_node.west) -| (add_node.north);

    
    \node[right=1.0cm of output] (split) {};
    
    \node[draw=orange, fill=orange!5, rounded corners, right=0.5cm of split, yshift=0.05cm] (loss_task) {
        \textbf{Task Loss}  $\min \mathcal{L}_{CE}$
    };
    
    \node[draw=red, fill=red!5, rounded corners, right=0.5cm of split, yshift=-0.6cm, font=\normalsize] (loss_def) {
        \textbf{Defense Loss}  $\max D_{KL}$
    };
    
    \draw[->, gray, dashed] (output.east) -- (loss_task.west);
    \draw[->, gray, dashed] (output.east) -- (loss_def.west);

    \node[above=0.1cm of extractor, font=\bfseries\color{green!40!black}\large] {Module 1: Structure-Aware Gating};
    \node[below=0.1cm of gnn_layer, font=\bfseries\color{blue!40!black}\large] {Module 2: Reprogramming Backbone};

    \end{tikzpicture}
    \caption{\textbf{Architecture of the GraphRP Framework.} The system operates via two coupled pathways. The \textit{Structure Branch} (Top) computes a gating factor $\alpha$ by matching the input graph against learned benign prototypes. The \textit{Reprogramming Branch} (Bottom) performs standard message passing but injects an effective bounded perturbation $\tilde{\boldsymbol{\delta}}^{(l)}$ at each layer. The injection is modulated by $\alpha$, effectively switching the defense \textit{ON} for \textit{OOD queries} and \textit{OFF} for \textit{benign queries}.}
    \label{fig:architecture_detail}
\end{figure*}

\section{Methodology}
\label{sec:method}

We propose \textbf{GraphRP} (\textbf{Graph} \textbf{R}eprogramming \textbf{P}rotection), a proactive defense framework that leverages model reprogramming to secure GNNs against ME attacks. 
The core intuition is to embed a dormant "security protocol" within the model's layers that activates only when it detects suspicious, OOD query patterns characteristic of extraction attempts.

\subsection{Overview of Pipeline}
\label{sec:method::overview}
The overall architecture of GraphRP is illustrated in Figure \ref{fig:architecture_detail}. 
Standard GNNs (left) process all inputs identically, making them vulnerable to attackers who query the model with synthetic or OOD graphs to approximate the decision boundary. 
In contrast, GraphRP (right) introduces two key components into the standard message-passing backbone:

\begin{enumerate}
    \item \textbf{Layer-Wise Reprogramming (Sec. \ref{sec:method::reprogram}):} A mechanism to inject learnable, task-specific perturbations into the latent node embeddings. This allows us to manipulate the model's output distribution without retraining the original weights.
    \item \textbf{Structure-Aware Gating (Sec. \ref{sec:method::structure}):} A dynamic control module that assesses the structural legitimacy of an input graph. It computes a gating factor $\alpha \in [0, 1]$ to modulate the reprogramming intensity.
\end{enumerate}

The pipeline operates in two modes:
\begin{itemize}
    \item \textbf{Benign Mode ($\alpha \to 0$):} For inputs matching the training distribution (ID), the reprogramming noise is suppressed. The model functions as a standard GNN, preserving high utility.
    \item \textbf{Defense Mode ($\alpha \to 1$):} For suspicious inputs (OOD), the gating mechanism activates the reprogramming layers. The model outputs are subtly distorted to maximize the divergence from the true decision boundary, effectively poisoning the attacker's surrogate training data.
\end{itemize}

\subsection{Layer-Wise Reprogramming}
\label{sec:method::reprogram}

The first challenge is to alter the behavior of a GNN without degrading its original performance. We adopt a \textit{Model Reprogramming} approach \cite{chen2024model}, treating the defense as a secondary task learned alongside the primary classification task.
Consider a standard GNN layer $l$ defined by the message passing operation:
\begin{align}
    \mathbf{H}^{(l)} = \sigma \left( \hat{\mathbf{A}} \mathbf{H}^{(l-1)} \mathbf{W}^{(l)} \right),
\end{align}
where $\mathbf{H}^{(l)} \in \mathbb{R}^{n \times d_l}$ are the node embeddings and $\hat{\mathbf{A}}$ is the normalized adjacency matrix.
A learnable perturbation $\mathbf{P}^{(l)}$ is injected into the latent space for defensive flexibility. 
However, applying a static noise matrix is insufficient as graph sizes vary. Instead, we define a \textbf{Universal Reprogramming Prototype} $\boldsymbol{\delta}^{(l)} \in \mathbb{R}^{d_l}$ for each layer, which is broadcasted to all nodes.

The reprogrammed layer output $\tilde{\mathbf{H}}^{(l)}$ is formulated as:
\begin{align}
    \label{eqn:reprogram_layer}
    \tilde{\mathbf{H}}^{(l)} = \mathbf{H}^{(l)} + \underbrace{\lambda \cdot \operatorname{Tanh}(\boldsymbol{\delta}^{(l)})}_{\text{Bounded Perturbation}},
\end{align}
where $\boldsymbol{\delta}^{(l)}$ are the learnable parameters of the defense. The $\operatorname{Tanh}(\cdot)$ function ensures the perturbation remains bounded, preventing numerical instability.
Crucially, $\boldsymbol{\delta}^{(l)}$ captures the \textit{direction} in the latent space that is most disruptive to an attacker's learning process. 
If we optimized $\boldsymbol{\delta}^{(l)}$ solely to maximize prediction error, it would destroy utility. Thus, the magnitude of this perturbation must be controlled by the graph structure, which we detail next.

\subsection{Structure-Aware Prototype Gating}
\label{sec:method::structure}

A naive application of the perturbation in Eq. \eqref{eqn:reprogram_layer} treats benign and malicious queries equally. To distinguish them, we propose a \textbf{Structure-Aware Prototype Gating} mechanism. 
Instead of relying on simple heuristics (like average degree), we learn a set of "Benign Structural Prototypes" in the embedding space.

\textbf{1. Structural Projection.}
We first map the input graph $G$ to a permutation-invariant structural embedding $\mathbf{z}_G$. We utilize a lightweight \textit{Readout} function $\Phi(\cdot)$ that aggregates node degrees, clustering coefficients, and spectral features (top-$k$ Laplacian eigenvalues):
\begin{align}
    \mathbf{z}_G = \Phi(G) = \left[ \text{MLP}\left( \sum_{v \in V} \left[\text{deg}(v) \parallel C(v)\right] \right) ; \text{Eig}(\mathbf{L}) \right].
\end{align}
Here, $C(v)$ denotes the local clustering coefficient of node $v$, $\parallel$ represents feature concatenation, $[\cdot\,;\,\cdot]$ denotes vector concatenation, and $\text{Eig}(\mathbf{L})$ extracts the top-$k$ eigenvalues of the normalized Laplacian \cite{chung1997spectral, defferrard2016convolutional}.
This projection captures the topological "signature" of the graph.

\textbf{2. Benign Prototypes.}
During training, we maintain a set of $K$ centroids $\{\mathbf{c}_k\}_{k=1}^K$ initialized via K-Means on $\mathcal{D}_{id}$ to represent the structural clusters of the private in-distribution data. The prototypes are refined in a dedicated structural update step (independently of $\Theta_{rep}$) and are explicitly excluded from the reprogramming gradient updates in Step 4 of Algorithm~\ref{alg:defense}.

\textbf{3. Gating Factor Calculation.}
For a new query $G_q$, we calculate its distance to the nearest benign prototype. The gating factor $\alpha(G_q)$ is derived using a soft-thresholding function:
\begin{align}
    d_{min} &= \min_{k} \| \mathbf{z}_{G_q} - \mathbf{c}_k \|_2^2, \\
    \alpha(G_{q}) &= \operatorname{Sigmoid}\left( \gamma \cdot (d_{min} - \tau) \right),  
    \label{eqn:gating_factor}
\end{align}
where $\tau$ is a learnable distance threshold and $\gamma$ is a sharpness parameter. If $d_{min} < \tau$ (structurally similar to private data), then $\alpha \to 0$.
If $d_{min} > \tau$ (structurally anomalous/OOD), then $\alpha \to 1$.

\textbf{Final Reprogramming Equation.}
Combining Eq. \eqref{eqn:reprogram_layer} and the gating factor, the final operation at layer $l$ becomes:
\begin{align}
    \tilde{\mathbf{H}}^{(l)} = \mathbf{H}^{(l)} + \alpha(G) \cdot \lambda \cdot \operatorname{Tanh}(\boldsymbol{\delta}^{(l)}).
\end{align}
This ensures that the defensive noise is strictly conditional on the structural "foreignness" of the query, effectively creating a graph-structure-sensitive firewall.

\textbf{Robustness against Structural Spoofing.} A natural question may arise: \textit{Can an adaptive attacker forge graph queries to bypass this gating mechanism? }
Unlike adding Gaussian noise to images, generating synthetic graphs that match the complex topological signature (e.g., Laplacian spectrum and clustering coefficients) of a private distribution is a non-trivial inverse problem \cite{ying2018hierarchical}. 
Without access to the benign training set, the attacker cannot easily estimate the centroid prototypes $\mathbf{c}_k$. 
Thus, our spectral gating effectively acts as a ``structural firewall,'' forcing the attacker to query OOD samples that trigger the defense.

\subsection{Defensive Optimization and Algorithm}
\label{sec:method::algorithm}

Having defined the structure-aware reprogramming mechanism, we now formulate the defensive training as an optimization problem. The defender's goal is to learn the optimal reprogramming parameters $\Theta_{rep} = \{\boldsymbol{\delta}^{(l)}, \tau\}_{l}$ that selectively degrade the fidelity of the attacker's surrogate model without compromising benign accuracy. Note that the prototypes $\{\mathbf{c}_k\}$ are maintained separately and excluded from $\Theta_{rep}$.

\subsubsection{Loss Function Formulation}
The total objective function $\mathcal{L}$ is a weighted combination of three terms: the utility preservation loss, the extraction defense loss, and a structural compactness loss.

\textbf{1. Utility Preservation ($\mathcal{L}_{\text{task}}$).} 
For in-distribution (ID) data, the reprogramming layer should have minimal impact. We minimize the standard Cross-Entropy loss on the private training set $\mathcal{D}_{id}$:
\begin{align}
    \label{eqn:task_loss}
    \mathcal{L}_{\text{task}} = \mathbb{E}_{(x, y) \sim \mathcal{D}_{id}} \left[ l_{\text{CE}}(T(x; \theta_T, \Theta_{rep}), y) \right].
\end{align}
Since the gating factor $\alpha(x) \to 0$ for ID data, this term ensures the perturbations remain dormant for benign users.

\textbf{2. Extraction Defense ($\mathcal{L}_{\text{defense}}$).}
For out-of-distribution (OOD) queries—which act as a proxy for attack queries—we explicitly maximize the divergence between the reprogrammed model's output and the original model's output. This "poisons" the information retrieved by the attacker. We utilize the Kullback-Leibler divergence:
\begin{align}
    \label{eqn:defense_loss}
    \mathcal{L}_{\text{defense}} = \mathbb{E}_{x \sim \mathcal{D}_{ood}} \left[ \max(0, \mu - D_{\text{KL}}(T(x; \theta_T) \parallel T(x; \theta_T, \Theta_{rep}))) \right].
\end{align}
Here, we employ a hinge loss formulation with margin $\mu$. This encourages the reprogrammed output to deviate from the original prediction by at least margin $\mu$, effectively obfuscating the true decision boundary for OOD inputs. 

By minimizing this hinge loss, the optimization encourages a large KL divergence between the target and reprogrammed distributions, which serves a dual purpose.
Beyond simple output distortion, it implicitly maximizes the \textbf{Fisher Information (FI)} distance along the direction of the perturbation $\boldsymbol{\delta}$.
As we will prove in Theorem \ref{thm:graph_attacker}, maximizing this term forces the attacker's surrogate model to optimize in the region of highest curvature in the loss landscape, thereby maximizing the estimation error for the stolen model parameters.

\textbf{3. Structural Compactness ($\mathcal{L}_{\text{struct}}$).}
To ensure the "Benign Prototypes" $\mathbf{c}_k$ accurately represent the ID distribution, we add a regularization term that minimizes the distance between ID samples and their nearest prototype:
\begin{align}
    \label{eqn:struct_loss}
    \mathcal{L}_{\text{struct}} = \mathbb{E}_{x \sim \mathcal{D}_{id}} \left[ \min_{k} \| \Phi(x) - \mathbf{c}_k \|_2^2 \right].
\end{align}
This creates tighter clusters for the benign data, improving the sensitivity of the gating mechanism $\alpha(x)$.

The final optimization problem is to find $\Theta_{rep}^*$ such that:
\begin{align}
    \label{eqn:total_loss}
    \Theta_{rep}^* = \arg\min_{\Theta_{rep}} \left( \mathcal{L}_{\text{task}} + \beta_1 \cdot \mathcal{L}_{\text{defense}} + \beta_2 \cdot \mathcal{L}_{\text{struct}} \right),
\end{align}
where $\beta_1$ and $\beta_2$ are hyperparameters controlling the defense strength and structural clustering, respectively. Note that we freeze the original model parameters $\theta_T$ to ensure parameter efficiency.

\subsubsection{Training Algorithm}
The complete training procedure for GraphRP is summarized in Algorithm~\ref{alg:defense}. We use an alternating update strategy: first minimizing the structural loss to refine the gating boundary, and then optimizing the reprogramming noise to enforce the defense.

\begin{algorithm}[t]
    \caption{Defensive Training for \textbf{GraphRP}.}
    \label{alg:defense}
    \begin{algorithmic}[1]
        \STATE \textbf{Input:} Pre-trained Victim Model $T(\cdot; \theta_T)$, Private Data $\mathcal{D}_{id}$, Auxiliary OOD Data $\mathcal{D}_{ood}^{def}$.
        \STATE \textbf{Hyperparameters:} Margins $\mu$, weights $\beta_1, \beta_2$, learning rate $\eta$.
        \STATE \textbf{Initialize:} Reprogramming noise $\boldsymbol{\delta}^{(l)} \sim \mathcal{N}(0, \epsilon)$, Prototypes $\{\mathbf{c}_k\}$ via K-Means on $\mathcal{D}_{id}$ \textit{(updated in Step 2 only; excluded from $\Theta_{rep}$ gradient updates)}, Distance threshold $\tau$.
        
        \WHILE{not converged}
            \STATE \textbf{Step 1: Data Sampling}
            \STATE Sample benign batch $\mathcal{B}_{id} = \{(x, y)\} \sim \mathcal{D}_{id}$.
            \STATE Sample OOD batch $\mathcal{B}_{ood} = \{x'\} \sim \mathcal{D}_{ood}^{def}$.
            
            \STATE \textbf{Step 2: Structural Update}
            \STATE Compute embeddings $Z = \Phi(\mathcal{B}_{id})$.
            \STATE Update $\{\mathbf{c}_k\}$ to minimize $\mathcal{L}_{\text{struct}}$ with Eq.~\eqref{eqn:struct_loss} \textit{(independent of $\Theta_{rep}$ gradient)}.
            
            \STATE \textbf{Step 3: Gating \& Forward Pass}
            \STATE Calculate gates $\alpha(x)$ for $x \in \mathcal{B}_{id} \cup \mathcal{B}_{ood}$ with Eq. \eqref{eqn:gating_factor}.
            \STATE Compute reprogrammed outputs $\tilde{y} = T(x; \theta_T, \Theta_{rep})$.
            \STATE Compute original outputs $y_{orig} = T(x; \theta_T)$ (for KL).
            
            \STATE \textbf{Step 4: Reprogramming Update}
            \STATE Calculate $\mathcal{L}_{\text{task}}$ on $\mathcal{B}_{id}$ with Eq. \eqref{eqn:task_loss}.
            \STATE Calculate $\mathcal{L}_{\text{defense}}$ on $\mathcal{B}_{ood}$ with Eq. \eqref{eqn:defense_loss}.
            \STATE Compute Total Loss $\mathcal{L}$ with Eq.~\eqref{eqn:total_loss}.
            \STATE Update $\Theta_{rep} \leftarrow \Theta_{rep} - \eta \nabla_{\Theta_{rep}} \mathcal{L}$.
        \ENDWHILE
        
        \STATE \textbf{Return:} Protected Model $T_{protected}(\cdot) = T(\cdot; \theta_T, \Theta_{rep}^*)$.
    \end{algorithmic}
\end{algorithm}

\textbf{Complexity Analysis.} 
The proposed defense is highly efficient. The structural extraction $\Phi(x)$ relies on simple statistics (degree, eigenvalues) computable in $O(|E| + n^3)$ or approximated in linear time. The reprogramming layer involves only element-wise addition, adding negligible cost to the $O(|E|d)$ GNN inference. Importantly, since $\theta_T$ is frozen, the number of trainable parameters is extremely small ($|\Theta_{rep}| \ll |\theta_T|$), ensuring rapid convergence.

%% file: 5-analysis.tex
\section{Theoretical Analysis}
\label{sec:analysis}

In this section, we provide a theoretical analysis for the effectiveness of GraphRP. We consider the worst-case scenario where the attacker has infinite query budget and capacity, aiming to learn a clone $C$ that properly mimics the reprogrammed victim $T^R$. 
We measure the defense success by the \textit{Loss Disparity} $Q(C^*,T) \triangleq \mathcal{L}_{id}(C^*) - \mathcal{L}_{id}(T)$ on the benign distribution $\mathcal{G}_{id}$. 
Here $\mathcal{G}_{id}$ and $\mathcal{G}_{ood}$ denote the graph distributions induced by $\mathcal{D}_{id}$
and $\mathcal{D}_{ood}$, respectively.

Let $\tilde{\boldsymbol{\delta}}^{(l)} := \lambda \cdot \operatorname{Tanh}(\boldsymbol{\delta}^{(l)})$ denote the effective bounded perturbation at layer $l$, which satisfies $\|\tilde{\boldsymbol{\delta}}^{(l)}\| \leq \lambda$ by construction. Our analysis assumes: (A1) the loss $\ell$ is bounded by constant $M$; (A2) the attacker is optimal, i.e., $C^*$ perfectly mimics $T^R$ on $\mathcal{G}_{ood}$; and (A3) the KL divergence admits a local second-order (Fisher Information) approximation in the bounded perturbation regime, where $\|\alpha(G)\cdot\tilde{\boldsymbol{\delta}}\| \leq \lambda\cdot\alpha(G) \leq \lambda$ by construction.

\begin{theorem}[Structural Defense Bound]
\label{thm:graph_attacker}
Let $C^*$ be the optimal clone model trained on OOD queries $G \sim \mathcal{G}_{ood}$ reprogrammed by GraphRP. Under the cross-entropy loss, the performance gap between the clone and the victim on benign tasks is lower-bounded by:
\begin{align}
    Q(C^*, T) \ge \underbrace{\mathbb{E}_{G \sim \mathcal{G}_{ood}} \left[ \frac{1}{2} \alpha(G)^2 \cdot \tilde{\boldsymbol{\delta}}^\top \mathbf{I}_{\mathbf{h}}(G) \tilde{\boldsymbol{\delta}} \right]}_{\text{Structural Sensitivity Term}} - \underbrace{2M \cdot \mathbb{TV}(\mathcal{G}_{id}, \mathcal{G}_{ood})}_{\text{Distribution Shift Term}},
\end{align}
where $\alpha(G)$ is the structure-aware gating factor, $\tilde{\boldsymbol{\delta}} = \lambda\cdot\operatorname{Tanh}(\boldsymbol{\delta})$ is the effective bounded perturbation, and $\mathbf{I}_{\mathbf{h}}(G)$ is the Fisher Information Matrix of the victim GNN with respect to the latent representation $\mathbf{h}$.
\end{theorem}

\textbf{Remark.}
The term $\tilde{\boldsymbol{\delta}}^\top \mathbf{I}_{\mathbf{h}}(G)\tilde{\boldsymbol{\delta}}$
captures the \textbf{structural sensitivity} of the GNN, which is largest when the reprogramming noise aligns with high-curvature directions of the model.
Because $\tilde{\boldsymbol{\delta}}$ is bounded, the analysis remains in the local regime where the second-order approximation in (A3) applies.
For attack queries with $\alpha(G)\approx 1$, GraphRP drives the learned surrogate boundary away from the benign boundary in proportion to this sensitivity.
The full proof is given in Appendix~\ref{sec:appendix::proof}; the bound holds under assumptions (A1)--(A3) and should not be interpreted as unconditional.

%% file: 6-experiments.tex
\begin{table*}[t]
\caption{\textbf{Main Defense Performance.} Clone model accuracy (lower is better) under Model Extraction attacks. GraphRP consistently outperforms state-of-the-art baselines across diverse datasets.}
\label{tab:main_results}
\centering
\small
\setlength{\tabcolsep}{4pt}
\begin{tabular}{lccccccccc}
\toprule
\multirow{2}{*}{\textbf{Defense Method}} & \multicolumn{3}{c}{\textbf{MUTAG}} & \multicolumn{3}{c}{\textbf{ENZYMES}} & \multicolumn{3}{c}{\textbf{NCI1}} \\
\cmidrule(lr){2-4} \cmidrule(lr){5-7} \cmidrule(lr){8-10} 
 & G-SAGE & GIUNET & GIC & G-SAGE & GIUNET & GIC & G-SAGE & GIUNET & GIC \\ \midrule
Undefended & 0.765 & 0.934 & 0.904 & 0.561 & 0.685 & 0.609 & 0.635 & 0.787 & 0.820 \\
RandP \cite{orekondy2019prediction} & 0.734 & 0.852 & 0.835 & 0.527 & 0.632 & 0.572 & 0.603 & 0.762 & 0.799 \\
P-Poison \cite{orekondy2019prediction} & 0.742 & 0.884 & 0.875 & 0.522 & 0.641 & 0.581 & 0.603 & 0.765 & 0.802 \\
GRAD \cite{mazeika2022steer} & 0.735 & 0.872 & 0.823 & 0.520 & 0.626 & 0.573 & 0.595 & 0.753 & 0.791 \\ 
AM \cite{kariyappa2020defending} & 0.722 & 0.866 & 0.882 & 0.514 & 0.633 & 0.551 & 0.610 & 0.757 & 0.798 \\
MeCo \cite{wang2023defending} & 0.712 & 0.823 & 0.813 & 0.482 & 0.619 & 0.552 & 0.586 & 0.683 & 0.736 \\ \midrule
\textbf{GraphRP (Ours)} & \textbf{0.603} & \textbf{0.782} & \textbf{0.750} & \textbf{0.364} & \textbf{0.575} & \textbf{0.545} & \textbf{0.521} & \textbf{0.656} & \textbf{0.682} \\ 
\bottomrule
\end{tabular}
\end{table*}

\begin{table*}[ht]
\caption{Target model utility (test accuracy) and $\ell_1$ norm of the output difference.}
\small
\centering
\label{tab:utility-1}
\begin{tabular}{lcccccccc}
\toprule
\multirow{2}{*}{\textbf{Defense}} & \multicolumn{2}{c}{\textbf{MUTAG}} & \multicolumn{2}{c}{\textbf{ENZYMES}} & \multicolumn{2}{c}{\textbf{NCI1}} & \multicolumn{2}{c}{\textbf{PROTEINS}} \\ \cmidrule(lr){2-3} \cmidrule(lr){4-5} \cmidrule(lr){6-7} \cmidrule(lr){8-9} 
 & Accuracy $\uparrow$ & $\ell_1$ norm $\downarrow$ & Accuracy $\uparrow$ & $\ell_1$ norm $\downarrow$ & Accuracy $\uparrow$ & $\ell_1$ norm $\downarrow$ & Accuracy $\uparrow$ & $\ell_1$ norm $\downarrow$\\ \midrule
Undefended & 0.9452 & 0.0 & 0.6629 & 0.0 & 0.8453 & 0.0 & 0.8012 & 0.0 \\
RandP & 0.9250 & 1.0 & 0.6451 & 1.0 & 0.8115 & 1.0 & 0.7864 & 1.0 \\
P-Poison & 0.9324 & 1.0 & 0.6324 & 1.0 & 0.8134 & 1.0 & 0.7823 & 1.0 \\
GRAD & 0.9276 & 1.0 & 0.6467 & 1.0 & 0.8072 & 1.0 & 0.7825 & 1.0 \\
AM & \bf 0.9335 & 1.0 & 0.6352 & 1.0 & 0.8025 & 1.0 & \bf 0.7924 & 1.0 \\
MeCo &  0.9297 & 0.3572 &  0.6327 &  0.0649 &  0.8107 &  0.1240 & 0.7871  &  0.1956 \\
\midrule
\textbf{GraphRP (Ours)} & 0.9319 & \bf 0.2351 & \bf 0.6467 & \bf 0.0567 & \bf 0.8155 & \bf 0.0762 & 0.7885 & \bf 0.1320 \\ \bottomrule
\end{tabular}
\end{table*}

\section{Experiments}
\label{sec:expr}

In this section, we empirically validate GraphRP. We aim to answer three key questions:

(1) \textbf{Defense Effectiveness:} Does GraphRP significantly degrade the performance of surrogate models compared to state-of-the-art baselines?
(2) \textbf{Utility Preservation:} Does the defensive reprogramming maintain the utility of the target model for benign users?
(3) \textbf{Scalability \& Robustness:} Is the method effective on large-scale graph benchmarks and robust against adaptive attackers?

\subsection{Experimental Setup}
\label{sec:expr::settings}
\textbf{Datasets.} We evaluate on a diverse set of graph classification benchmarks \cite{morris2020tudataset}, ranging from biochemical molecules to social networks.
To demonstrate scalability, we include \textbf{NCI1} \cite{wale2008comparison} (4k+ graphs) and \textbf{OGB-MolHIV} \cite{hu2020ogb} (41k+ graphs), alongside standard baselines \textbf{MUTAG} \cite{debnath1991structure}, \textbf{ENZYMES} \cite{borgwardt2005protein}, and \textbf{PROTEINS} \cite{borgwardt2005protein}. Note that this work focuses on \textit{graph-level classification}, the dominant MLaaS deployment setting for GNNs.

\textbf{Baselines.} We compare GraphRP against five defense strategies: (1) \textbf{RandP} (Random Perturbation) \cite{orekondy2019prediction}, (2) \textbf{P-Poison} \cite{orekondy2019prediction}, (3) \textbf{GRAD} (Gradient Redirection) \cite{mazeika2022steer}, (4) \textbf{AM} (Adaptive Misinformation) \cite{kariyappa2020defending}, and (5) \textbf{MeCo} \cite{wang2023defending}, a recent contrastive defense. These baselines represent the full spectrum of active defense strategies: output obfuscation (RandP, P-Poison), gradient redirection (GRAD), adaptive misinformation (AM), and contrastive learning (MeCo), all applicable to the black-box graph classification setting.

\textbf{Attack Configuration.} We employ \textbf{KnockoffNet}~\cite{orekondy2019knockoff} as the query strategy. The attacker trains surrogate models (GraphSAGE~\cite{hamilton2017inductive}, Graph Isomorphism U-Net (GIUNET)~\cite{amouzad2024graph}, Gaussian-Induced Convolution (GIC)~\cite{jiang2019gaussian}) using either \textit{soft-label} (probability vector) or \textit{hard-label} (top-1) supervision from the victim. Our evaluation collectively covers soft-label attacks, hard-label attacks, adaptive gray-box attacks (Section~\ref{sec:expr::robustness}), data leakage scenarios, and varying query budgets, providing comprehensive coverage of realistic ME threat models.
Detailed hyperparameters and architecture settings are provided in Appendix \ref{sec:appendix::settings}.

\subsection{Defense Performance (RQ1 \& RQ2)}
\label{sec:expr::performance}

\subsubsection{Reducing Clone Model Accuracy}
The primary goal of the defender is to maximize the error of the attacker's clone model. The detailed numbers are shown in Table~\ref{tab:main_results}, Table~\ref{tab:clone-acc-1} (MUTAG \& ENZYMES) and Table~\ref{tab:clone-acc-2} (NCI1 \& PROTEINS).

Figure \ref{fig:main_results} summarizes the performance on MUTAG and ENZYMES. We can observe that GraphRP consistently achieves the lowest clone accuracy across all scenarios. 
For instance, on MUTAG under soft-label attack, our method reduces the clone accuracy from $76.5\%$ (Undefended) to $60.3\%$, significantly outperforming the best baseline (MeCo at $71.2\%$).

\begin{figure}[t]
    \centering
    \begin{tikzpicture}
    \begin{axis}[
        ybar,
        bar width=8pt,
        width=0.48\textwidth,
        height=5cm,
        ylabel={Clone Accuracy ($\downarrow$)},
        symbolic x coords={MUTAG-Soft, MUTAG-Hard, ENZ-Soft, ENZ-Hard},
        xtick=data,
        nodes near coords,
        nodes near coords style={font=\tiny, rotate=90, anchor=west},
        ymin=0.3, ymax=1.0,
        legend style={at={(0.5,1.15)}, anchor=north, legend columns=-1, font=\footnotesize},
        ymajorgrids=true,
        grid style=dashed,
        enlarge x limits=0.2,
        font=\sffamily\footnotesize
    ]
        \addplot[fill=gray!30, draw=none] coordinates {
            (MUTAG-Soft, 0.765) (MUTAG-Hard, 0.735) (ENZ-Soft, 0.561) (ENZ-Hard, 0.487)
        };
        \addplot[fill=blue!40, draw=none] coordinates {
            (MUTAG-Soft, 0.712) (MUTAG-Hard, 0.613) (ENZ-Soft, 0.482) (ENZ-Hard, 0.425)
        };
        \addplot[fill=red!60!black, draw=none] coordinates {
            (MUTAG-Soft, 0.603) (MUTAG-Hard, 0.572) (ENZ-Soft, 0.364) (ENZ-Hard, 0.342)
        };
        
        \legend{No Defense, MeCo, \textbf{GraphRP (Ours)}}
    \end{axis}
    \end{tikzpicture}
    \caption{\textbf{Attack Success Rate comparison.} Clone model accuracy (lower is better) on MUTAG and ENZYMES. GraphRP outperforms the best baseline (MeCo) and the undefended setting, reducing clone accuracy by up to $\sim 15\%$ compared to the undefended model.}
    \label{fig:main_results}
\end{figure}
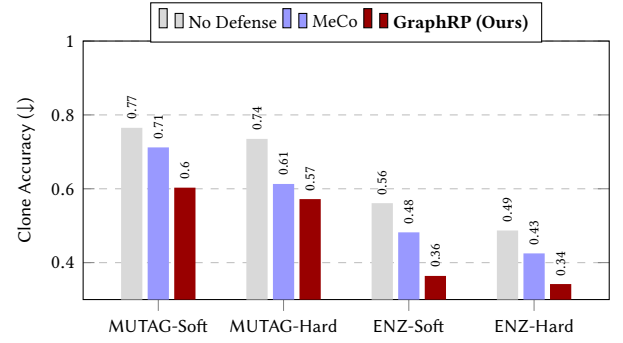

\subsubsection{Utility-Defense Trade-off}
A robust defense must not degrade the experience for benign users, as shown in Table~\ref{tab:utility-1}.

We analyze the trade-off between \textit{Utility Loss} and \textit{Defense Success}. As shown in Figure~\ref{fig:tradeoff}, GraphRP achieves the strongest defense while incurring less than $2\%$ utility loss, whereas baseline methods suffer a substantially larger utility-defense trade-off. This advantage stems from the \textit{Structure-Aware Gating} mechanism, which effectively distinguishes benign from adversarial queries (AUROC $>0.91$, FPR $<0.06$), enabling selective activation of the defense.

\textbf{Computational Efficiency.} A potential concern with spectral methods is the cost of eigendecomposition. 
However, GraphRP mitigates this by computing only the top-$k$ eigenvalues (where $k \ll n$) using efficient iterative methods like the Lanczos algorithm. 
Furthermore, feature extraction is performed at the \textit{graph level}, not the node level, meaning the overhead scales linearly with the batch size rather than the number of nodes. 
As shown in Table \ref{tab:inference-time}, the additional latency is marginal ($\approx 7\%$) compared to the heavy matrix multiplications in the GNN backbone, ensuring GraphRP remains suitable for real-time MLaaS latency constraints.

\subsection{Robustness and Scalability Analysis (RQ3)}
\label{sec:expr::robustness}

To validate the resilience of GraphRP in hostile environments, we conducted two stress tests (detailed results in Appendix \ref{sec:appendix::robustness}).

\textbf{Defending against Adaptive Attackers.} 
We simulated a sophisticated "gray-box" attacker who is aware of the GraphRP defense mechanism. 
This adaptive adversary employs a \textit{Structural Generator} (based on GraphGAN \cite{wang2018graphgan}) explicitly trained to generate queries that mimic the topological statistics of the benign dataset, attempting to bypass our Structure-Aware Gating.

As shown in Table~\ref{tab:adaptive_attack_app}, GraphRP largely remains effective even against this adaptive strategy, confirming that spectral eigenvalues serve as a ``structural fingerprint'' that is mathematically difficult to spoof without access to the private training manifold. Furthermore, GraphRP outperforms simpler OOD-based strategies---query rejection (OOD-Reject) and uniform-confidence return (Low-Conf)---achieving strictly lower clone accuracy and higher benign accuracy simultaneously, while returning outputs in a confidence range similar to genuine predictions.

Finally, Table~\ref{tab:id-ood-mix} shows that GraphRP maintains robustness even when the attacker gains access to a portion of the private training data: with $10\%$ ID data leaked, clone accuracy increases by only $3.2\%$ (vs.\ $5.2\%$ for Undefended), confirming resilience to partial distribution shift.

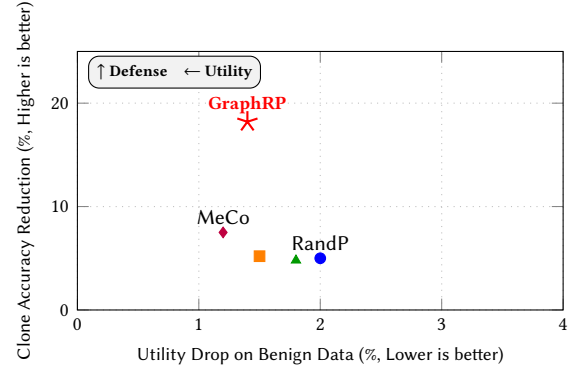
\begin{figure}
    \centering
    \begin{tikzpicture}
    \begin{axis}[
        width=0.45\textwidth,
        height=5cm,
        xlabel={Utility Drop on Benign Data ($\%$, Lower is better)},
        ylabel={Clone Accuracy Reduction ($\%$, Higher is better)},
        xmin=0, xmax=4,
        ymin=0, ymax=25,
        grid=both,
        grid style={dotted},
        legend pos=north east,
        font=\sffamily\footnotesize,
        scatter/classes={
            a={mark=*, blue},
            b={mark=square*, orange},
            c={mark=triangle*, green!60!black},
            d={mark=diamond*, purple},
            ours={mark=star, mark size=4pt, red, thick}
        }
    ]
        \addplot[scatter, only marks, scatter src=explicit symbolic] coordinates {
            (2.0, 5.0) [a]  
            (1.5, 5.2) [b]  
            (1.8, 4.8) [c]  
            (1.2, 7.5) [d]  
            (1.4, 18.2) [ours] 
        };
        
        \node[anchor=south] at (axis cs: 2.0, 5.0) {\small RandP};
        \node[anchor=south] at (axis cs: 1.2, 7.5) {\small MeCo};
        \node[anchor=south] at (axis cs: 1.4, 18.2) {\bf \color{red} GraphRP};
        
        \node[ draw,rounded corners, fill=gray!10, font=\bfseries\scriptsize] at (axis cs:0.8,23)
        {$\uparrow$ Defense \ \ \ $\leftarrow$ Utility};
    \end{axis}
    \end{tikzpicture}
    \caption{\textbf{Utility vs. Defense Trade-off on NCI1.} The y-axis represents the reduction in attacker accuracy (Benefit), and the x-axis represents the drop in benign accuracy (Cost). GraphRP provides the best protection-to-cost ratio.}
    \label{fig:tradeoff}
\end{figure}






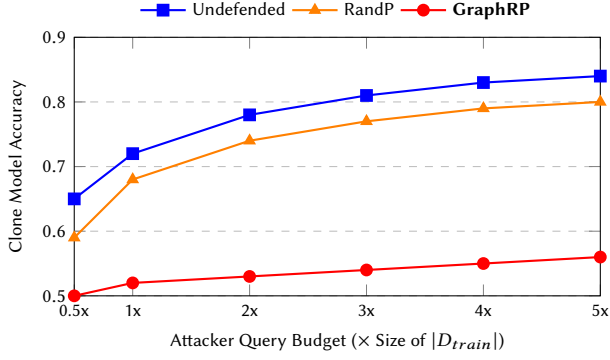
\begin{figure}
    \centering
    \begin{tikzpicture}
    \begin{axis}[
        width=0.48\textwidth,
        height=5cm,
        xlabel={Attacker Query Budget ($\times$ Size of $|D_{train}|$)},
        ylabel={Clone Model Accuracy},
        xmin=0.5, xmax=5.0,
        ymin=0.5, ymax=0.9,
        xtick={0.5, 1, 2, 3, 4, 5},
        xticklabels={0.5x, 1x, 2x, 3x, 4x, 5x},
        legend pos=north west,
        legend style={ at={(0.5,1.03)}, anchor=south, legend columns=3, draw=none},
        ymajorgrids=true,
        grid style=dashed,
        cycle list name=color list,
        font=\sffamily\footnotesize
    ]
        \addplot[color=blue, mark=square*, thick] coordinates {
            (0.5, 0.65) (1.0, 0.72) (2.0, 0.78) (3.0, 0.81) (4.0, 0.83) (5.0, 0.84)
        };
        \addlegendentry{Undefended}

        \addplot[color=orange, mark=triangle*, thick] coordinates {
            (0.5, 0.59) (1.0, 0.68) (2.0, 0.74) (3.0, 0.77) (4.0, 0.79) (5.0, 0.80)
        };
        \addlegendentry{RandP}

        \addplot[color=red, mark=*, thick] coordinates {
            (0.5, 0.50) (1.0, 0.52) (2.0, 0.53) (3.0, 0.54) (4.0, 0.55) (5.0, 0.56)
        };
        \addlegendentry{\textbf{GraphRP}}
        
    \end{axis}
    \end{tikzpicture}
    \caption{\textbf{Effect of Query Budget on Attack Success (NCI1).} As the attacker increases the query budget, the clone accuracy of the undefended model rises steadily. In contrast, GraphRP effectively saturates the attacker's learning, keeping accuracy low ($\approx 55\%$) regardless of the budget.}
    \label{fig:budget_trend}
\end{figure}

\begin{table}[t]
\caption{\textbf{Adaptive Defense Robustness (MUTAG).} Clone accuracy under Standard \& Adaptive attacks. GraphRP remains effective even when attacker mimics benign graph structure.}
\label{tab:adaptive_attack_app}
\centering
\small
\begin{tabular}{lcccc}
\toprule
\multirow{2}{*}{\textbf{Defense}} & \multicolumn{2}{c}{\textbf{Standard Attack}} & \multicolumn{2}{c}{\textbf{Adaptive Attack}} \\
\cmidrule(lr){2-3} \cmidrule(lr){4-5}
 & soft-label & hard-label & soft-label & hard-label \\ \midrule
Undefended & 0.765 & 0.735 & 0.781 & 0.752 \\
RandP & 0.734 & 0.701 & 0.762 & 0.738 \\
\midrule
\textbf{GraphRP} & \textbf{0.603} & \textbf{0.572} & \textbf{0.615} & \textbf{0.588} \\
\bottomrule
\end{tabular}
\end{table}

\begin{table}[t]
\caption{\textbf{Robustness against Data Leakage (MUTAG).} Clone accuracy when attacker mixes $10\%$ private ID data into queries. GraphRP remains robust under partial distribution shift.}
\label{tab:id-ood-mix}
\centering
\small
\begin{tabular}{lcc}
\toprule
\textbf{Defense} & \textbf{Standard (OOD Only)} & \textbf{Adaptive (10\% ID Leak)} \\ \midrule
Undefended & 0.7651 & 0.8174 \\
RandP & 0.7341 & 0.7763 \\
P-Poison & 0.7426 & 0.7695 \\
AM & 0.7223 & 0.7570 \\ \midrule
\textbf{GraphRP} & \textbf{0.6032} & \textbf{0.6358} \\
\bottomrule
\end{tabular}
\end{table}

\textbf{Scalability (Large-Scale Graphs).} Table~\ref{tab:large_scale_main} reports defense performance on two large-scale benchmarks: \textbf{OGB-MolHIV} (41k molecular graphs) and \textbf{COLLAB} (dense social network). GraphRP achieves the lowest clone accuracy on both datasets while maintaining high benign utility, demonstrating that our spectral structural prototypes generalize effectively to large and topologically diverse graph manifolds.

\begin{table}[t]
\caption{\textbf{Large-Scale Defense Performance.} Clone accuracy ($\downarrow$) and test accuracy ($\uparrow$) on OGB-MolHIV and COLLAB.}
\label{tab:large_scale_main}
\centering
\small
\setlength{\tabcolsep}{4pt}
\begin{tabular}{lcccc}
\toprule
\multirow{2}{*}{\textbf{Defense}} & \multicolumn{2}{c}{\textbf{OGB-MolHIV}} & \multicolumn{2}{c}{\textbf{COLLAB}} \\
\cmidrule(lr){2-3} \cmidrule(lr){4-5}
 & Clone Acc $\downarrow$ & Test Acc $\uparrow$ & Clone Acc $\downarrow$ & Test Acc $\uparrow$ \\ \midrule
Undefended & 0.7924 & 0.7150 & 0.6978 & 0.6987 \\
RandP      & 0.7645 & 0.6840 & 0.6742 & 0.6389 \\
P-Poison   & 0.7420 & 0.6910 & 0.6523 & 0.6420 \\
AM         & 0.7210 & 0.6980 & 0.6426 & 0.6524 \\ \midrule
\textbf{GraphRP} & \textbf{0.5840} & \textbf{0.7125} & \textbf{0.6245} & \textbf{0.6972} \\
\bottomrule
\end{tabular}
\end{table}

\textbf{Impact of Query Budget.} Standard active defenses often fail under unlimited query budgets, as noise eventually averages out. As shown in Figure~\ref{fig:budget_trend}, while the Undefended model's clone accuracy increases steadily with budget ($0.5\times$--$5\times$), GraphRP saturates at a low level ($\approx 55\%$ on NCI1), indicating that reprogramming noise fundamentally poisons the decision boundary, rendering additional queries uninformative.

\textbf{Defense Transferability against Complex Surrogates.} 
A realistic risk is that an attacker employs a more powerful architecture than the victim to extract the model. 
We tested the following scenario: the victim uses the standard G-Inception target, while the attacker uses a powerful Graph Transformer~\cite{dwivedi2020generalization}. As detailed in Appendix~\ref{sec:appendix::transfer}, GraphRP successfully degrades the Graph Transformer's performance from $82.4\%$ to $64.1\%$.
This confirms that GraphRP's layer-wise reprogramming distorts the underlying data manifold itself, making the poisoned labels effective regardless of the attacker's model capacity.

\textbf{Ablation.} To understand which features drive our Gating Mechanism, we ablate Node Degrees, Clustering Coefficients, and Spectral Eigenvalues. As shown in Figure~\ref{fig:struct_ablation}, \textbf{Spectral Features} are the most critical: removing them increases clone accuracy from 0.603 to 0.685, indicating a substantially weaker defense, while removing degree statistics has only a minor impact, validating that global spectral properties are the hardest for attackers to mimic.

\begin{figure}[t]
    \centering
    \begin{tikzpicture}
    \begin{axis}[
        ybar,
        bar width=20pt,
        width=0.48\textwidth,
        height=5cm,
        ylabel={Clone Accuracy (Lower is Better)},
        symbolic x coords={FullMethod, woDegree, woClustering, woSpectral},
        xtick={FullMethod, woDegree, woClustering, woSpectral},
        xticklabels={Full Method, w/o Degree, w/o Clustering, w/o Spectral},
        nodes near coords,
        nodes near coords style={font=\tiny, color=black},
        ymin=0.4, ymax=0.8,
        ymajorgrids=true,
        grid style=dashed,
        enlarge x limits=0.15,
        font=\sffamily\footnotesize,
        bar shift=0pt 
    ]
        \addplot[fill=red!60!black] coordinates {(FullMethod, 0.603)};
        \addplot[fill=orange!80] coordinates {(woDegree, 0.615)};
        \addplot[fill=yellow!80!black] coordinates {(woClustering, 0.642)};
        \addplot[fill=gray!60] coordinates {(woSpectral, 0.685)};
    \end{axis}
    \end{tikzpicture}
    \caption{\textbf{Feature Importance Ablation (MUTAG).} Removing spectral features causes the sharpest accuracy rise, confirming they are the most critical component for OOD detection.}
    \label{fig:struct_ablation}
\end{figure}
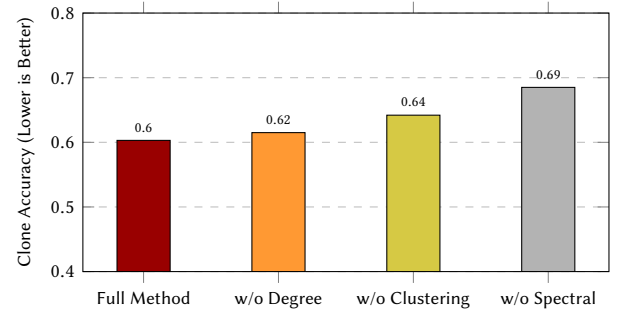

%% file: 7-conclusions.tex
\section{Conclusion}
\label{sec:conclusion}

In this work, we presented \textbf{GraphRP}, a dynamic defense against GNN model extraction.
GraphRP combines model reprogramming with structure-aware gating to selectively poison surrogate training while preserving benign utility.
Under the stated assumptions, our analysis provides a lower bound on the attacker's estimation error, and experiments show that GraphRP scales effectively to large graph datasets.

\paragraph{Future Directions.}
Future work will extend GraphRP to node-level tasks through localized ego-graph prototypes and adapt it to heterogeneous and dynamic graphs, where evolving topology requires continuous prototype calibration.

%% file: appendix.tex
\setcounter{table}{0}
\renewcommand{\thetable}{A\arabic{table}}

\section{Details of Theoretical Analysis}
\label{sec:appendix::proof}

Here, we provide the rigorous derivation of Theorem \ref{thm:graph_attacker}.

\subsection{Preliminaries}
Let $\mathcal{G}_{id}$ and $\mathcal{G}_{ood}$ denote the benign and adversarial graph distributions. Let $\tilde{\boldsymbol{\delta}}^{(l)} := \lambda \cdot \operatorname{Tanh}(\boldsymbol{\delta}^{(l)})$ denote the effective bounded perturbation at layer $l$, satisfying $\|\tilde{\boldsymbol{\delta}}^{(l)}\| \leq \lambda$ by construction. The victim model output is $T(G)$, and the reprogrammed victim is $T^R(G) = T(\mathbf{H}^{(l)} + \alpha(G)\tilde{\boldsymbol{\delta}})$ for the $l$-th layer. The attacker minimizes the expected risk on the query distribution: $\mathcal{L}_{ood}(C) = \mathbb{E}_{G \sim \mathcal{G}_{ood}} [\ell(C(G), T^R(G))]$. We assume the loss function $\ell$ is bounded by constant $M$.

\subsection{Proof of Theorem \ref{thm:graph_attacker}}

\begin{proof}

\textbf{Step 1: Relating ID and OOD Performance.}
We aim to lower-bound $Q(C^*, T) = \mathcal{L}_{id}(C^*) - \mathcal{L}_{id}(T)$.
For any bounded function $f$ with bound $M$, the difference in expected loss across distributions satisfies $|\mathbb{E}_P[f] - \mathbb{E}_Q[f]| \le M \cdot \|P-Q\|_1$, which gives:
\begin{equation}
    |\mathcal{L}_{id}(C^*) - \mathcal{L}_{ood}(C^*)| \le 2M \cdot \mathbb{TV}(\mathcal{G}_{id}, \mathcal{G}_{ood}).
\end{equation}
Since the victim $T$ is well-trained on $\mathcal{D}_{id}$ (i.e., $\mathcal{L}_{id}(T) \approx 0$):
\begin{equation}
    Q(C^*, T) \ge \mathcal{L}_{ood}(C^*) - 2M \cdot \mathbb{TV}(\mathcal{G}_{id}, \mathcal{G}_{ood}).
\end{equation}

\textbf{Step 2: Optimal Attacker Assumption.}
Assume the attacker trains the clone $C^*$ to minimize the divergence from the observed query responses $T^R(G)$.
Under the optimal attacker assumption, the clone perfectly learns the reprogrammed distribution, i.e., $C^*(G) \approx T^R(G)$.
We evaluate the clone's error $\mathcal{L}_{ood}$ against the \textit{original} victim $T$ using Cross-Entropy.
By decomposing Cross-Entropy into Entropy ($H$) and KL-Divergence ($D_{\text{KL}}$), and noting that $H(T(G)) \ge 0$, we obtain the lower bound:
\begin{align}
    \mathcal{L}_{ood}(C^*)
    &= \mathbb{E}_{G \sim \mathcal{G}_{ood}} [ H(T(G)) + D_{\text{KL}}( T(G) \parallel C^*(G) ) ] \nonumber\\
    &\ge \mathbb{E}_{G \sim \mathcal{G}_{ood}} [ D_{\text{KL}}( T(G) \parallel T^R(G) ) ].
\end{align}

\textbf{Step 3: Taylor Expansion via Fisher Information.}
Let $P(Y|\mathbf{h})$ be the victim's output distribution parameterized by graph embedding $\mathbf{h}$, so $T^R(G)$ corresponds to $P(Y|\mathbf{h} + \epsilon)$ with $\epsilon = \alpha(G) \cdot \tilde{\boldsymbol{\delta}}$.
Since $\tilde{\boldsymbol{\delta}} = \lambda \cdot \operatorname{Tanh}(\boldsymbol{\delta})$, the effective perturbation satisfies $\|\epsilon\| \leq \lambda \cdot \alpha(G) \leq \lambda$, which ensures the perturbation remains in the local regime where the following Taylor expansion is applicable.
Since $D_{\text{KL}}(P\|P) = 0$ and the first-order term vanishes, a second-order Taylor expansion gives:
\begin{equation}
    D_{\text{KL}}( P(Y|\mathbf{h}) \parallel P(Y|\mathbf{h} + \epsilon) ) \approx \tfrac{1}{2} \epsilon^\top \mathbf{I}_{\mathbf{h}}(G) \epsilon + O(\|\epsilon\|^3),
\end{equation}
where $\mathbf{I}_{\mathbf{h}}(G) = \mathbb{E}_{y \sim P(Y|\mathbf{h})} \left[ \nabla_{\mathbf{h}} \log P(y|\mathbf{h}) \nabla_{\mathbf{h}} \log P(y|\mathbf{h})^\top \right]$ is the Fisher Information Matrix (FIM).

\textbf{Step 4: Combining Terms.}
Substituting $\epsilon = \alpha(G) \cdot \tilde{\boldsymbol{\delta}}$:
\begin{equation}
    D_{\text{KL}}( T(G) \parallel T^R(G) ) \approx \frac{1}{2} \alpha(G)^2 \tilde{\boldsymbol{\delta}}^\top \mathbf{I}_{\mathbf{h}}(G) \tilde{\boldsymbol{\delta}}.
\end{equation}
Substituting this back into the bounds from Steps~1 and~2 yields the final bound:
\begin{equation}
    Q(C^*, T) \ge \mathbb{E}_{G \sim \mathcal{G}_{ood}} \left[ \frac{1}{2} \alpha(G)^2 \tilde{\boldsymbol{\delta}}^\top \mathbf{I}_{\mathbf{h}}(G) \tilde{\boldsymbol{\delta}} \right] - 2M \cdot \mathbb{TV}(\mathcal{G}_{id}, \mathcal{G}_{ood}).
\end{equation}
\end{proof}

\begin{table*}[h]
\caption{\textbf{Hard-label defense performance on MUTAG \& ENZYMES.} Soft-label results are in Table~\ref{tab:main_results}.}
\small
\centering
\label{tab:clone-acc-1}
\begin{tabular}{lcccccc}
\toprule
\multirow{2}{*}{\textbf{Defense}} & \multicolumn{3}{c}{\textbf{MUTAG} Clone Model Architecture} & \multicolumn{3}{c}{\textbf{ENZYMES} Clone Model Architecture} \\ 
\cmidrule(lr){2-4} \cmidrule(lr){5-7} 
 & GraphSAGE & GIUNET & GIC & GraphSAGE & GIUNET & GIC \\ 
 \midrule
Undefended $\downarrow$ & 0.7346 & 0.8835 & 0.8657 & 0.4874 & 0.6557 & 0.5880 \\
RandP $\downarrow$ & 0.7012 & 0.8087 & 0.7564 & 0.4475 & 0.5821 & 0.5517 \\
P-Poison $\downarrow$ & 0.7089 & 0.8231 & 0.8054 & 0.4682 & 0.5967 & 0.5458 \\
GRAD $\downarrow$ & 0.7120 & 0.8193 & 0.7901 & 0.4626 & 0.5919 & 0.5530 \\
AM $\downarrow$ & 0.6957 & 0.8125 & 0.7627 & 0.4587 & 0.5830 & 0.5462 \\
MeCo $\downarrow$ & 0.6135 & 0.7974 & 0.7456 & 0.4251 & 0.5724 & 0.5407 \\
\textbf{GraphRP (Ours)} $\downarrow$ & \textbf{0.5721} & \textbf{0.7531} & \textbf{0.7238} & \textbf{0.3421} & \textbf{0.5447} & \textbf{0.5223} \\ 
\bottomrule
\end{tabular}
\end{table*}

\begin{table*}[t]
\caption{\textbf{Hard-label defense performance on NCI1 \& PROTEINS.} Soft-label results are in Table~\ref{tab:main_results}.}
\small
\centering
\label{tab:clone-acc-2}
\begin{tabular}{lcccccc}
\toprule
\multirow{2}{*}{\textbf{Defense}} & \multicolumn{3}{c}{\textbf{NCI1} Clone Model Architecture} & \multicolumn{3}{c}{\textbf{PROTEINS} Clone Model Architecture} \\ 
\cmidrule(lr){2-4} \cmidrule(lr){5-7} 
 & GraphSAGE & GIUNET & GIC & GraphSAGE & GIUNET & GIC \\ \midrule
Undefended $\downarrow$ & 0.6073 & 0.7425 & 0.7824 & 0.6951 & 0.7164 & 0.7135 \\
RandP $\downarrow$ & 0.5735 & 0.7144 & 0.7459 & 0.6592 & 0.6455 & 0.6820 \\
P-Poison $\downarrow$ & 0.5752 & 0.7120 & 0.7634 & 0.6536 & 0.6837 & 0.6852 \\
GRAD $\downarrow$ & 0.5731 & 0.7146 & 0.7661 & 0.6492 & 0.6902 & 0.6813 \\
AM $\downarrow$ & 0.5675 & 0.7235 & 0.7653 & 0.6473 & 0.6923 & 0.6793 \\
MeCo $\downarrow$ & 0.5435 & 0.6946 & 0.6837 & 0.5864 & 0.6771 & 0.6527 \\
\midrule
\textbf{GraphRP (Ours)} $\downarrow$ & \textbf{0.5024} & \textbf{0.6137} & \textbf{0.6547} & \textbf{0.5003} & \textbf{0.6287} & \textbf{0.6325} \\  
\bottomrule
\end{tabular}
\end{table*}

\begin{table*}[h]
\caption{Clone model accuracy after applying \textit{adaptive attack} on MUTAG with \text{G\_Inception} as target model.}
\label{tab:adaptive-attack}
\centering
\small
\begin{tabular}{@{}llccc@{}}
\toprule
\multirow{2}{*}{\textbf{Attack}} & \multirow{2}{*}{\textbf{Defense}} & \multicolumn{3}{c}{\textbf{Clone Model Architecture}} \\ \cmidrule(l){3-5} 
 &  & \multicolumn{1}{c}{GraphSAGE} & \multicolumn{1}{c}{GIUNET} & \multicolumn{1}{c}{GIC} \\ 
 \midrule
\multirow{4}{*}{\begin{tabular}[c]{@{}l@{}}Hard-label \\ Attack\end{tabular}}
 & Undefended $\downarrow$ & 0.7651 & 0.9342 & 0.9043 \\
 & \bf GraphRP $\downarrow$ & 0.6032 & 0.7829 & 0.7506 \\
 & \bf GraphRP, Adaptive, unknown architecture $\downarrow$ & \textbf{0.5220} & \textbf{0.6334} & \textbf{0.6101} \\ 
 & \bf GraphRP, Adaptive, known architecture $\downarrow$ & 0.5725 & 0.6672 & 0.6502 \\ 
 \midrule
\multirow{4}{*}{\begin{tabular}[c]{@{}l@{}}Soft-label\\ Attack\end{tabular}}
 & Undefended $\downarrow$ & 0.7346 & 0.8835 & 0.8657 \\
 & \bf GraphRP $\downarrow$ & 0.5721 & 0.7531 & 0.7238 \\
 & \bf GraphRP, Adaptive, unknown architecture $\downarrow$ & \textbf{0.5023} & \textbf{0.5942} & \textbf{0.5731} \\ 
 & \bf GraphRP, Adaptive, known architecture $\downarrow$ & 0.5495  & 0.6247 &  0.5986 \\ 
 \bottomrule
\end{tabular}
\end{table*}

\section{Implementation Details of GraphRP}
\label{sec:appendix::implementation}

\subsection{Structure Extractor $\Phi(G)$}

The structural embedding $\mathbf{z}_G = \Phi(G)$ concatenates three permutation-invariant feature groups: (1) \textbf{Spectral Features} — top-$k$ sorted eigenvalues of the normalized Laplacian $\mathbf{L} = \mathbf{I} - \mathbf{D}^{-1/2}\mathbf{A}\mathbf{D}^{-1/2}$, with zero-padding for small graphs ($n < 50$) and truncation for large ones ($k=8$ for molecular datasets); (2) \textbf{Degree Statistics} — a 4-dimensional vector of $[\min, \max, \text{mean}, \text{std}]$ of node degrees; (3) \textbf{Global Properties} — graph diameter, average clustering coefficient, and density. The final vector is $\ell_2$-normalized: $\mathbf{z}_G \leftarrow \mathbf{z}_G / \|\mathbf{z}_G\|_2$.

\subsection{Prototype Learning}

Prototypes $\{\mathbf{c}_k\}_{k=1}^K$ are initialized via K-Means ($K=5$) on $\mathcal{D}_{id}^{train}$ and updated exclusively via $\mathcal{L}_{\text{struct}}$ in Step~2 of Algorithm~\ref{alg:defense}, independently of $\Theta_{rep}$. The gating factor is:
\begin{equation}
    \alpha(G) = \operatorname{Sigmoid}\left( \gamma \cdot \left(\min_k \|\Phi(G) - \mathbf{c}_k\|_2^2 - \tau\right) \right),
\end{equation}
where $\gamma=10$ and $\tau$ is a trainable threshold initialized at the $95^{\text{th}}$ percentile of benign distances.

\section{Extended Experimental Settings}
\label{sec:appendix::settings}

\subsection{Hardware and Hyperparameters}
All models are implemented in PyTorch $\ge$ 1.12 with PyTorch Geometric. Experiments were conducted on a cluster with NVIDIA RTX A5000 (24GB VRAM) and RTX 6000 Ada (48GB) GPUs.

We use G-Inception (3 layers, 128 dimensional hidden layers) as the target model for biochemical datasets and GraphSAGE (3 layers) for social datasets. All models are optimized using Adam, with the reprogramming noise $\boldsymbol{\delta}$ trained at learning rate $1\times10^{-3}$ and weight decay $5\times10^{-4}$. The query budget is set to $20\times$ the size of the benign training set for hard-label attacks and $10\times$ for soft-label attacks.

The code for our experiments can be accessed at \url{https://github.com/overwenyan/GraphRP-KDD2026/}.

\begin{table}[h]
\caption{Inference Time Comparison.}
\label{tab:inference-time}
\small
\centering
\begin{tabular}{lc} 
\toprule
\textbf{Defense Method} & \textbf{Inference Time (s)} \\ 
\midrule
Undefended     & 52.31  \\
RandP          & 54.72  \\
P-Poison       & 432.17 \\
AM             & 115.57 \\
\textbf{GraphRP (Ours)}  & \textbf{56.42}  \\ 
\bottomrule
\end{tabular}
\end{table}

\subsection{Inference Latency Analysis}
A critical requirement for MLaaS defense is low latency. As shown in Table \ref{tab:inference-time}, GraphRP incurs a marginal overhead of only $\approx 7.9\%$ compared to the Undefended model (56.42s vs. 52.31s). 
This slight increase is primarily due to the spectral decomposition in $\Phi(G)$, which is computed once per graph. 
In contrast, optimization-based methods like \textit{P-Poison} are computationally prohibitive, requiring iterative gradient steps at test time that make them over $8\times$ slower (432.17s) than the baseline. 
GraphRP's efficient forward-pass design ensures it remains compatible with real-time applications.

\section{Extended Robustness and Transferability}
\label{sec:appendix::robustness}

\subsection{Adaptive Attack and Data Leakage}
\textbf{Defending against Adaptive Attackers.} As shown in Tables~\ref{tab:adaptive_attack_app} and~\ref{tab:adaptive-attack}, GraphRP remains robust against a gray-box attacker (using GraphGAN to mimic benign topology \cite{wang2018graphgan}), limiting clone accuracy to $61.5\%$ (soft-label) and $58.8\%$ (hard-label)---a negligible increase ($< 1.5\%$) over the standard attack. Notably, even when the attacker knows the defended architecture, clone accuracy stays below $65\%$.


\subsection{Cross-Architecture Transferability}
\label{sec:appendix::transfer}
As shown in Table~\ref{tab:transferability}, GraphRP is architecture-agnostic: layer-wise reprogramming distorts the data manifold regardless of surrogate capacity, reducing extraction accuracy by over $18\%$ even against a powerful Graph Transformer~\cite{dwivedi2020generalization}.

\begin{table}[t]
\caption{\textbf{Defense Transferability (MUTAG).} Clone accuracy across attacker architectures. GraphRP's poisoning effect holds even against Graph Transformers.}
\label{tab:transferability}
\centering
\small
\begin{tabular}{lccc}
\toprule
\textbf{Victim: G-Inception} & \multicolumn{3}{c}{\textbf{Attacker Architecture}} \\\cmidrule(l){2-4}
Defense Method & GraphSAGE & GAT & Graph Transformer \\ \midrule
Undefended & 0.765 & 0.789 & 0.824 \\
RandP (Baseline) & 0.734 & 0.751 & 0.792 \\
\textbf{GraphRP (Ours)} & \textbf{0.603} & \textbf{0.612} & \textbf{0.641} \\ 
\bottomrule
\end{tabular}
\end{table}